\documentclass[letterpaper, 10 pt, conference]{ieeeconf}  

\IEEEoverridecommandlockouts                              
{}
{}
{}

\usepackage{graphicx}        
\usepackage{multicol}        
\usepackage[bottom]
           {footmisc}        
\usepackage{etoolbox}        
\usepackage[vlined,ruled,linesnumbered]
           {algorithm2e}
\usepackage{amsmath}                      
\usepackage{amsfonts}                     
\usepackage{amssymb}
\usepackage{amsthm}
\usepackage{mathrsfs}                     
\usepackage{float} 
\usepackage{graphicx} 
\usepackage{nicefrac}
\usepackage{ calc }           
\usepackage[usenames,dvipsnames,svgnames,table]
           {xcolor}
\usepackage[font=footnotesize]
           {caption}
\usepackage[font=footnotesize]
           {subcaption}
\usepackage{flushend}

\title{\LARGE \bf
Free-configuration Biased Sampling for Motion Planning:\\ Errata}

\author{Joshua Bialkowski,  Michael Otte, and Emilio Frazzoli}
\theoremstyle{definition}
\newtheorem{theorem}{Theorem}
\newtheorem{lemma}[theorem]{Lemma}
\newtheorem{corollary}[theorem]{Corollary}
\newtheorem{proposition}[theorem]{Proposition}

\theoremstyle{definition}

\makeatletter
\newcommand{\removelatexerror}{\let\@latex@error\@gobble}
\makeatother

\newlength{\tightalgowidth}
\newlength{\tightalgoremainder}

\newenvironment{tightalgo}[2][]
{ 
   \setlength{\tightalgowidth}{#2}
   \setlength{\tightalgoremainder}{\linewidth-\tightalgowidth}
   \begin{algorithm}[#1]
   \small }
{ \end{algorithm} }

\makeatletter
\patchcmd{\@algocf@start}{%
  \begin{lrbox}{\algocf@algobox}%
}{%
  \rule{0.5\tightalgoremainder}{\z@}%
  \begin{lrbox}{\algocf@algobox}%
  \begin{minipage}{\tightalgowidth}%
}{}{}
\patchcmd{\@algocf@finish}{%
  \end{lrbox}%
}{%
  \end{minipage}%
  \end{lrbox}%
}{}{}
\makeatother

\mathchardef\mhyphen="2D
\mathchardef\mh="2D

\newcommand{ \Fn }   [2]{  #1 \left( #2 \right) }

\newcommand{ \Proc  }[1]{ \mathtt{ #1 } }

\newcommand{ \Graph           }{ G }
\newcommand{ \EdgeSet         }{ E }
\newcommand{ \VertexSet       }{ V }

\newcommand{ \Configuration      }{  \Point }

\newcommand{ \SampledConfiguration }{  \Configuration_\mathrm{sample} }

\newcommand{ \NearVertices         }{  \VertexSet_\mathrm{near}       }

\newcommand{ \SetOf }[1]{ \left\{ #1 \right\} }
\newcommand{ \DefinedAs }{ \overset{\triangle}{=} }
\newcommand{ \SuchThat } { \, | \, }
\newcommand{\thereExists}{ \, \exists \, }

\newcommand{ \Node       }{ v }
\newcommand{ \NodeSet    }{ \mathcal{V} }
\newcommand{ \NumSampFn  }[2]{ N_{ #2 } \left( #1 \right) }

\newcommand{ \Point        }{ x }
\newcommand{ \SampledPoint }{ X }
\newcommand{ \QueryPoint   }{ \Point_q }
\newcommand{ \PointSet     }{ \mathcal{X} } 
\newcommand{ \HyperRect    }{ H }

\newcommand{\PointCoord}[2]{#1[#2]}
\newcommand{ \NumSamples }{ T                   }
\newcommand{ \NumFree    }{ F                   }

\newcommand{ \Parent     }{ P                   }

\newcommand{ \RealSet    }{ \mathbb{R}          }

\newcommand{ \Dimension  }{ d                   }

\newcommand{ \deep       }{ D                   }

\newcommand{ \Generate        }{ \mathtt{GenerateSample}    }

\newcommand{ \NodeParent  }{P}
\newcommand{ \NodeNumTotal }{T}
\newcommand{ \NodeNumFree  }{F}
\newcommand{ \NodeMeasure }{M}

\newcommand{ \NodeParentFn   }[1]{ {#1}.\NodeParent }
\newcommand{ \NodeNumTotalFn }[1]{ {#1}.\NodeNumTotal }
\newcommand{ \NodeNumFreeFn  }[1]{ {#1}.\NodeNumFree }
\newcommand{ \NodeMeasureFn  }[1]{ {#1}.\NodeMeasure }

\newcommand{\Union}     {\cup}
\newcommand{\UnionOver}[1]{\bigcup_{#1}}
\newcommand{\Intersect}{\cap}

\newcommand{\free}{\mathrm{F}}
\newcommand{\obs}{\mathrm{O}}

\newcommand{\obsSpace}{S_{\obs}}
\newcommand{\freeSpace}{S_{\free}}
\newcommand{\totalSpace}{S}
\newcommand{\probDensity}[2]{f_#1\left(#2\right)}
\newcommand{\probUniformFree}[1]{f_{\free}\left(#1\right)}

\newcommand{\probability}[1]{\mathbb{P}\left(#1\right)}

\newcommand{\node}{\Node}
\newcommand{\nodeset}{\NodeSet}

\newcommand{\spaceOf}[1]{\mathbb{S}\left(#1\right)}
\newcommand{\closedSpaceOf}[1]{ \overline{ \mathbb{S} }\left(#1\right)}

\newcommand{\obsNode}{\node^{\obs}}

\newcommand{\freeNodeSet}{\nodeset^{\free}}
\newcommand{\obsNodeSet}{\nodeset^{\obs}}
\newcommand{\mixedNodeSet}{\nodeset^{\mathrm{M}}}

\newcommand{\leafNode}{\node_{\mathrm{L}}}

\newcommand{\leafNodeSet}{\nodeset^{\mathrm{L}}}
\newcommand{\mixedLeafNodeSet}{\nodeset^{\mathrm{ML}}}
\newcommand{\obsLeafNodeSet}{\nodeset^{\mathrm{OL}}}
\newcommand{\freeLeafNodeSet}{\nodeset^{\mathrm{FL}}}

\newcommand{\measure}[1]{\mathscr{L}\left(#1\right)}

\newcommand{ \ChildParticular }{ c }
\newcommand{ \ChildSet }{ C }

\newcommand{\distMass}{\NodeMeasure}

\newcommand{\freeDescendents}{\hat{C}_n}

\newcommand{\epsBoxPoint}{\Xi_{\epsilon,\Point}}

\begin{document}

%
%
\maketitle

\thispagestyle{empty}
\pagestyle{empty}

\begin{abstract} 
This document contains improved and updated proofs of convergence for the sampling method presented in our paper ``Free-configuration Biased Sampling for Motion Planning'' \cite{bialkowski_IROS13}\\

\noindent The following is the abstract of the original paper:\\

In sampling-based motion planning algorithms the initial step at
every iteration is to generate a new sample from the obstacle-free portion of 
the configuration space. This is usually accomplished via rejection sampling, 
i.e., repeatedly drawing points from the entire space until an obstacle-free 
point is found. This strategy is rarely questioned because the extra 
work associated with sampling (and then rejecting) useless points contributes at 
most a constant factor to the planning algorithm's asymptotic runtime complexity. 
However, this constant factor can be quite large in practice.
We propose an alternative approach that enables sampling from a 
distribution that provably converges to a uniform distribution over {\it only} the obstacle-free space. Our method works by storing empirically observed estimates of obstacle-free space in a point-proximity data structure, and then using this 
information to generate future samples. Both theoretical and experimental 
results validate our approach.

\end{abstract}

\section{Introduction}

This document contains an improved and updated version of the proofs of convergence that appeared in our paper titled ``Free-configuration Biased Sampling for Motion Planning'' \cite{bialkowski_IROS13} as part of the proceeding of the 2013 IEEE/RSJ International Conference on Intelligent Robots and Systems (IROS). We have included only the material from the original paper that is necessary to understand the updated proofs. Therefore, we encourage the reader to also read the original paper \cite{bialkowski_IROS13}; it contains sections on motivation, related work, experiments, results, and runtime analysis---all of which do not appear in the current document.

\section{Algorithm} \label{sec:algorithm}

\begin{figure}
\centering
Sampling Distribution Induced by Our Algorithm in 2D

\vspace{.2cm} 

\begin{minipage}{0.24\linewidth} 
   \centering
   \includegraphics[width=\linewidth]{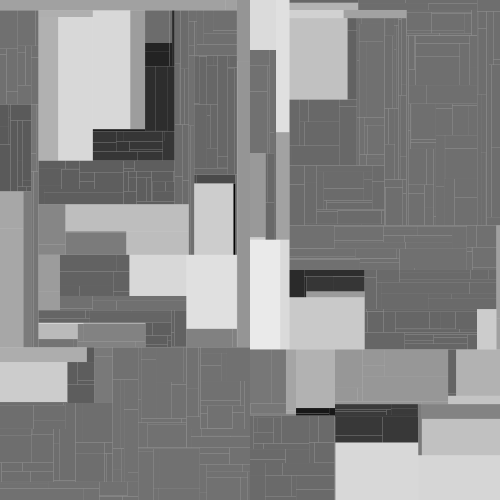}
   (a)
\end{minipage}
\begin{minipage}{0.24\linewidth}
   \centering
   \includegraphics[width=\linewidth]{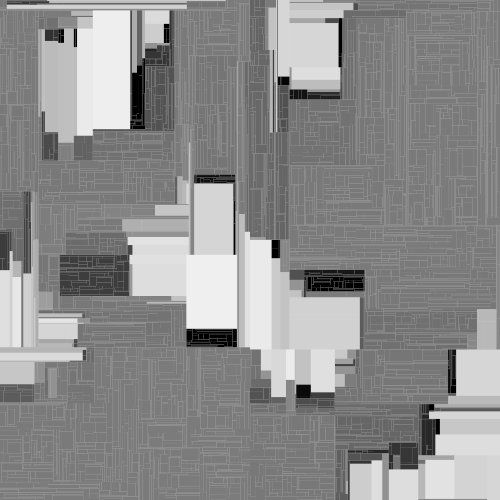}
   (b)
\end{minipage}
\begin{minipage}{0.24\linewidth}
   \centering
   \includegraphics[width=\linewidth]{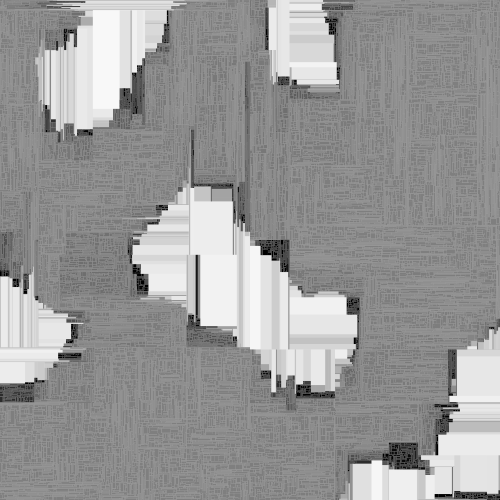}
   (c)
\end{minipage}
\begin{minipage}{0.24\linewidth}
   \centering
   \includegraphics[width=\linewidth]{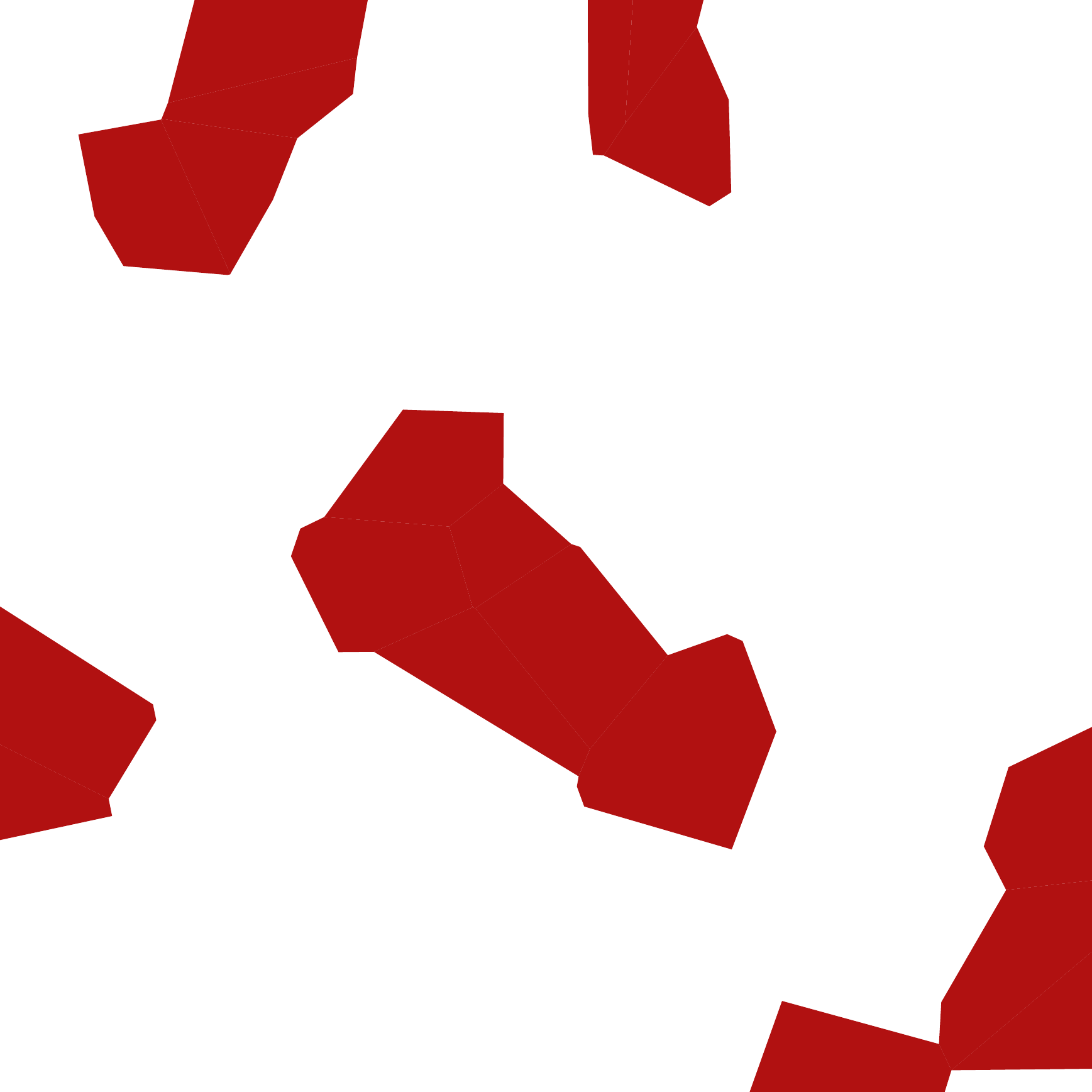}
   (d)
\end{minipage}
\caption{%
The induced sampling distribution of an augmented kd-tree after 
$10^4$, $10^5$, and $10^6$ samples are shown in (a), (b), and (c),
 respectively.
White-black represent low-high sampling probability density. 
The actual obstacle configuration appears in (d), obstacles are red.}
\label{fig:kdDemo}
\end{figure}

Our sampling algorithm is intended for use with sampling based motion planning.
Let $\Graph = \left( \VertexSet, \EdgeSet \right)$
denote the output of the motion planning algorithm, where $\Graph$ is a graph defined by its vertex set $\VertexSet$ and edge set $ \EdgeSet$.
At each iteration,
a new sample $\SampledConfiguration$ is generated from the obstacle-free 
space. A set of candidate nodes 
$\NearVertices \subset \VertexSet$ is selected  for possible
connection with $\SampledConfiguration$. That is, for each 
$\Configuration \in \NearVertices$, if a local planner determines a 
collision-free path exists from $\Configuration$ to $
\SampledConfiguration$, then $\SampledConfiguration$ is 
added to $\VertexSet$ and $[\Configuration , \SampledConfiguration]$ 
is added to $\EdgeSet$.

Our method relies on storing extra data in each node of a kd-tree. 
%
%
A kd-tree is a special type of binary search tree that can efficiently determine 
the nearest neighbor(s) of a query point $\QueryPoint$ 
within a previously defined finite set of points ${\PointSet \subset \RealSet^\Dimension}$ 
\cite{Bentley.acm75}.
Each node in the kd-tree is a {\tt Node} data structure, 
the fields of which are summarized in table \ref{table:NodeStruct}. 
%

Each node $\Node$ in a kd-tree defines an axis-aligned hyper-rectangle 
$\HyperRect(\Node) \subset \RealSet^\Dimension$. 
An interior node $v$ is 
assigned a point $\Point \in \PointSet \cap H(v)$ and an index $j \in \SetOf{1 \ldots d}$. 
Its two children are the hyper-rectangles found by splitting $\HyperRect(\Node)$ with 
a hyperplane passing through $\Point$ and orthogonal to the $j$-th axis.
Leaf nodes are the same as interior nodes except that they are not assigned a  
point and have no children (yet). 
Finally, for any $\HyperRect \subset \mathbb{R}^d$, $\Fn{\Proc{Measure}}{\Node}$ returns the measure of the set  $\HyperRect(\Node)$, and $\Fn{\Proc{SampleUniform}}{ \HyperRect }$ returns a point drawn from a uniform distribution over $\HyperRect$.

\begin{table}
   \centering
   \caption{ The {\tt Node} data structure }
   \begin{tabular}{l|l|p{1.9in}}
      \textbf{field}
         & \textbf{type}
         & \textbf{description}  \\
      \hline
      \hline
      $\Point$
         & vector $\in \RealSet^\Dimension$
         & point associated with this node \\
      $j$
         & integer $\in \SetOf{1..\Dimension}$
         & index of the split plane \\
      $c[2]$
         & {\tt Node} array
         & references to the children of the node (two in the case of a kd-%
           tree), or a \emph{null} reference $\emptyset$ if this is a leaf node\\
      $\Parent$
        & {\tt Node}
        & reference to the parent node \\
      $\NumSamples$
         & float $\in \RealSet$
         & weighted number of samples generated from $\HyperRect$ \\
      $\NumFree$
         & float $\in \RealSet$
         & weighted number of collision free samples generated from $\HyperRect$ \\
      $\distMass$
         & float $\in \RealSet$
         & estimated measure of free space in $\HyperRect$ 
   \end{tabular}
   \label{table:NodeStruct}
\end{table}

\begin{tightalgo}[t]{0.95\linewidth}
   \caption{ $\Fn{\Generate}{\HyperRect,\Node}$ } 
   \label{algo:sample}
   \eIf{ $\node.c[0] = \node.c[1] = \emptyset$ }
   {
      $\Point \leftarrow \Fn{\Proc{SampleUniform}}{ \HyperRect }$ \;
      $\node.\NumSamples \leftarrow \node.\NumSamples + 1$ \;
      $r = \Fn{\Proc{Collision\mh free}}{\Point}$ \; 
      \If{$r$}
      {
        $\node.\Point \leftarrow \Point$ \;
        $\node.\NumFree \leftarrow \node.\NumFree + 1$ \;
        $(\node.c[0],\node.c[1]) \leftarrow \Fn{\Proc{Split}}{\node,\Point}$ \; 
    \For{$i = \{0, 1\}$}
    {
        $\node.c[i].\Parent \leftarrow \Node $ \;
        $\node.c[i].j \leftarrow (\Node.j + 1) \bmod{\Dimension} $ \;
        $w \leftarrow \Fn{\Proc{Measure}}{\node.c[i]} / \Fn{\Proc{Measure}}{\node}$ \;
        $\node.c[i].\NumSamples \leftarrow w \cdot v.\NumSamples $ \;
        $\node.c[i].\NumFree \leftarrow w \cdot v.\NumFree $ \;
        $\node.c[i].\distMass = \left(\frac{\node.c[i].\NumFree }{ \node.c[i].\NumSamples} \right) \Fn{\Proc{Measure}}{ \Node.c[i] }$ \;
        }
      }
   }
   {
      $u \leftarrow \Fn{\Proc{SampleUniform}}{ [0,\node.\distMass] }$ \;
      \eIf{ $u \le \node.c[0].\distMass $ }
      {
         $(\Point,r) \leftarrow \Fn{\Generate}{ \Node.c[0] }$ \;
      }
      {
        $(\Point,r) \leftarrow \Fn{\Generate}{ \Node.c[1] }$ \;
      }
      
      $\node.\distMass = \node.c[0].\distMass + \node.c[1].\distMass$ \;
   }
   \Return $(\Point,r)$
\end{tightalgo} 

%
Note that we store three additional fields in each
node of our our augmented kd-tree: $\NumSamples$, $\NumFree$, and $\distMass$. Both $\NumSamples$ and $\NumFree$ are only used by leaf nodes. 
$\NumSamples$ is the total number of samples taken from $\HyperRect$, and 
$\NumFree$ is the number of those samples that are collision free. When a leaf node generates a new sample (and thus creates its children), each child inherits a weighted version of $\NumSamples$ and $\NumFree$ from the parent. Both values are weighted by the relative measure of space contained in the child vs.\ the parent, and both account for the successful sample before weighting (lines 3, 7-14).
$\distMass$ is our estimate of the measure of free space contained in $\HyperRect$. For leaf nodes, $\distMass = \frac{\NumFree}{\NumSamples} \Fn{\Proc{Measure}}{\HyperRect}$, line 14. For non-leaf nodes, $\distMass$ is the cumulative sum of the values of $\distMass$ contained in the node's children $\distMass = c[0].\distMass + c[1].\distMass$, line 21.

See our original IROS paper \cite{bialkowski_IROS13} for a more comprehensive description of the algorithm.

The procedure generates new samples by starting at the root, and then recursively picking a child based on a random coin flip weighted by the amount of free-space believed to be in a child. Once a leaf node is reached, the sample point is drawn from a uniform distribution over that leaf's hyper-rectangle.
The result of the collision check are used to propagate new weight information propagated back up the tree.

\section{Analysis} \label{sec:analysis}

We now prove that the sampling distribution induced by our algorithm converges to a uniform distribution over the free space. 

Let $\obsSpace$, $\freeSpace$, and $\totalSpace$ denote the obstacle space, free
space, and total space respectively, where  $\totalSpace = \obsSpace \cup \freeSpace$
and $\obsSpace \cap \freeSpace = \emptyset$.
We use the notation $\spaceOf{\cdot}$ to denote the subset of space associated with a data structure element, e.g., ${\HyperRect=\spaceOf{\node}}$ is the hyper-rectangle of $\node$.
We also use $\probability{\cdot}$ to denote probability, and $\measure{\cdot}$ 
to denote the Lebesgue measure, e.g.,
$\measure{\freeSpace}$ is the hyper-volume of the free space.
We assume that the configuration space is bounded 
and that the boundaries of 
$\obsSpace$, $\freeSpace$, and $\totalSpace$ have measure zero. 

Let $\ChildSet$ denote the set of children of $\Node$.
Each child ${\ChildParticular_i \in \ChildSet}$ represents a subset of ${\spaceOf{\node}}$ such that ${\bigcup_{i} \spaceOf{\ChildParticular_i} = \spaceOf{\node}}$ and ${\spaceOf{\ChildParticular_i} \cap \spaceOf{\ChildParticular_j} = \emptyset}$ for all ${i \neq j}$. In a kd-tree $|\ChildSet| = 2$.

{\it Note the wording in this section is tailored to the kd-tree version of our algorithm; however, the analysis is generally applicable to any related data-structure\footnote{In particular, we only require a tree-based space partitioning spatial index that is theoretically capable of containing any countably infinite set of points $\PointSet$, and such that the hyper-space of the leaf nodes covers the configuration space $\totalSpace = \UnionOver{\Node \in \leafNodeSet} \spaceOf{\Node}$. Our proofs can be modified to the general case by replacing `hyper-rectangle' with `hyper-space' and assuming that a weighted die determines the recursion path (instead of a coin). When the die is thrown at $\Node$ it has $|\ChildSet|$ sides and the weight of the $i$-th side is determined by the estimated value of $\measure{\spaceOf{\ChildParticular_i}}/ \measure{\spaceOf{\Node}}$ (i.e. the relative amount of free space believed to exist in child $\ChildParticular_i$ vs.\ its parent $\Node$).}.}

Let $\probDensity{n}{\cdot}$ be the probability density function for the sample returned  by Algorithm~\ref{algo:sample}, when the kd-tree contains $n$ points.
Let $\probUniformFree{\cdot}$ represent a probability density function such that $\probUniformFree{x_a}
= \probUniformFree{x_b}$ for all $x_a,x_b \in \freeSpace$ and
${\probUniformFree{x_c} = 0}$ for all ${x_c \in \obsSpace}$.
Let $\SampledPoint_{\free}$ and $\SampledPoint_n$ denote random variables drawn from the distributions 
defined by $\probUniformFree{\cdot}$ and $\probDensity{n}{\cdot}$, respectively.

We now prove that  
$\probability{\lim_{n \to \infty} \probDensity{n}{x} = \probUniformFree{x}}=1$, for almost all $x \in S$, 
%
i.e., that the induced distribution of our algorithm converges to a distribution that is almost surely equal to $\probUniformFree{x}$ almost everywhere in $S$, possibly excluding a measure-zero subset.

We begin by observing that the nodes in the kd-tree may be classified into three sets:
\begin{itemize}
\item free nodes, $\freeNodeSet$, the set of nodes $\Node$ such that 
   $\measure{\spaceOf{\Node} \Intersect \obsSpace} = 0$ and 
   $\thereExists \Point \SuchThat \Point \in \spaceOf{\Node} 
   \land \Point \in \freeSpace $.
\item obstacle nodes, $\obsNodeSet$, the set of nodes, $\Node$ such that 
${\measure{\spaceOf{\obsNode} \Intersect \freeSpace} = 0}$ and 
$\thereExists \Point \SuchThat \Point \in \spaceOf{\Node} 
   \land \Point \in \obsSpace$
\item mixed nodes, $\mixedNodeSet$ contains all nodes do not fit the definition of a free node or a mixed node.
\end{itemize}


Although our algorithm is ignorant of the type of any given node
(otherwise we would not need it to begin with), an oracle would
know that free nodes contain free space almost everywhere, obstacle nodes contain
obstacle space almost everywhere, and mixed nodes contain both free space and obstacle space. Note that if $\Node \in \mixedNodeSet$ and $\measure{\spaceOf{\Node}} > 0$ then 
$\measure{\spaceOf{\Node} \Intersect \obsSpace} > 0$ and 
$\measure{\spaceOf{\Node} \Intersect \freeSpace} > 0$. It is possible for $\node$ such that $\measure{\spaceOf{\node}}=0$ to be both an obstacle node and a free node if it exists on the the boundary between $\obsSpace$ and $\freeSpace$; because their cumulative measure is zero, such nodes can be counted as both obstacle nodes and free nodes (or explicitly defined as either one or the other) without affecting our results.

\begin{figure}[!t]
\begin{minipage}[t]{0.48\linewidth}
  \centering
  \includegraphics[width=\linewidth]{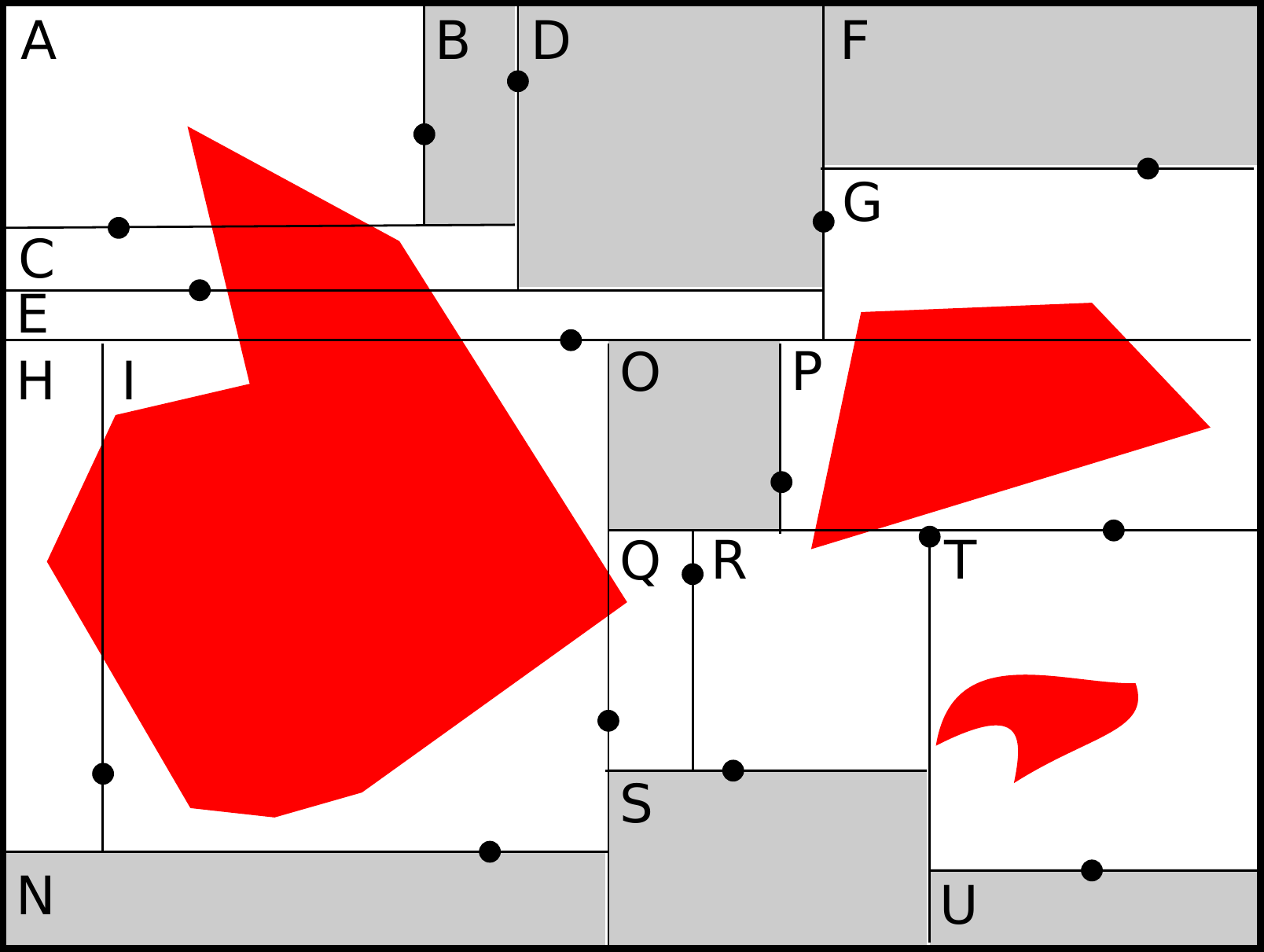}

  \vspace{.5cm}

  \includegraphics[width=\linewidth]{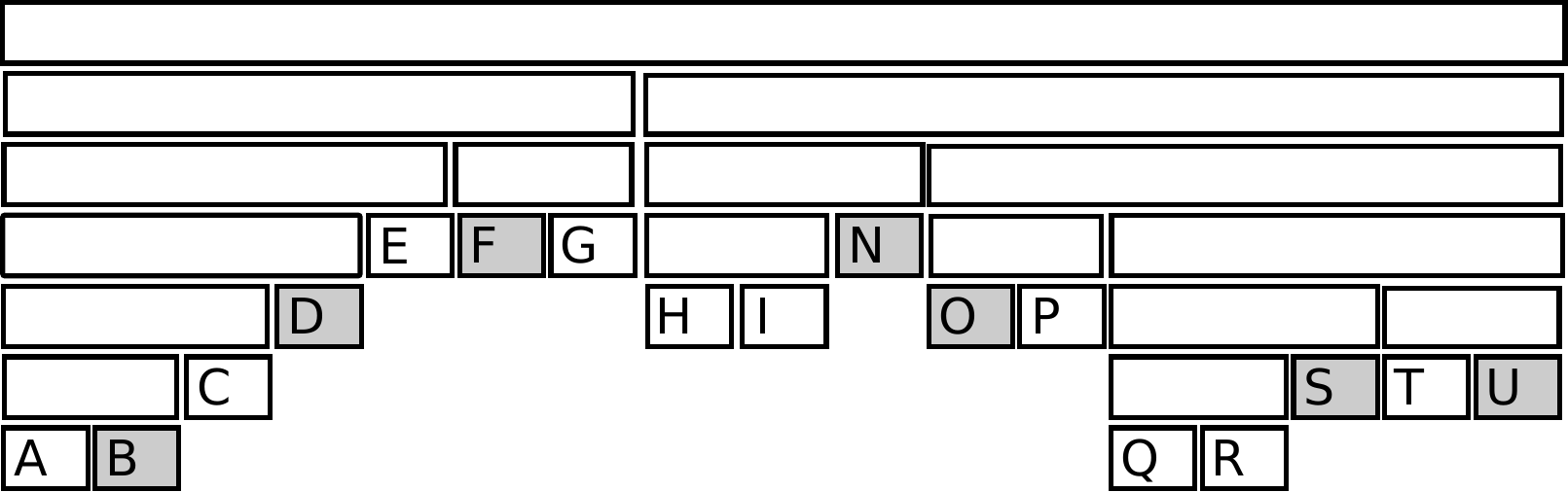}
\end{minipage}
\hspace{0.02\linewidth}
\begin{minipage}[t]{0.48\linewidth}
  \centering
  \includegraphics[width=\linewidth]{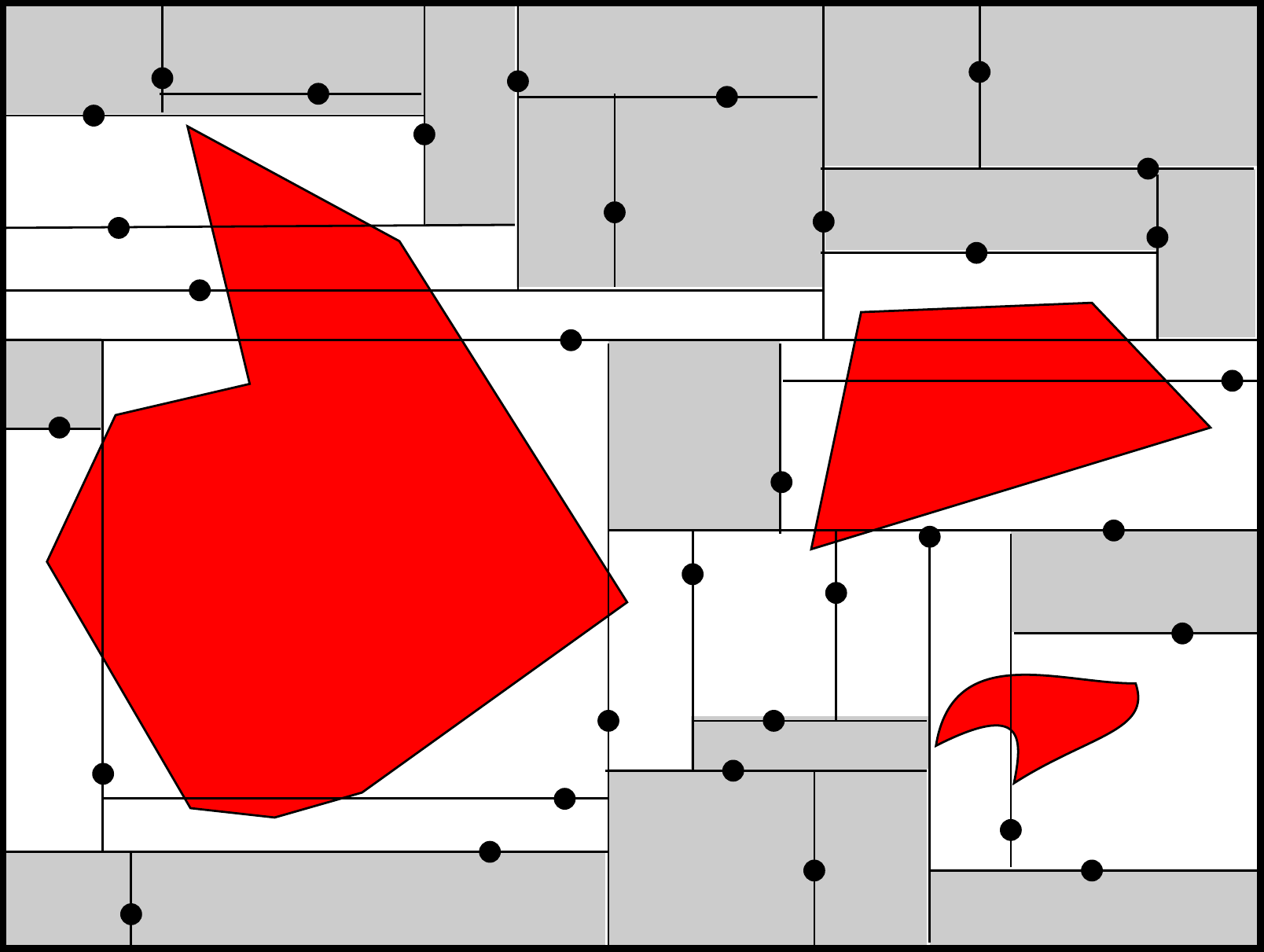}

  \vspace{.5cm}

  \includegraphics[width=\linewidth]{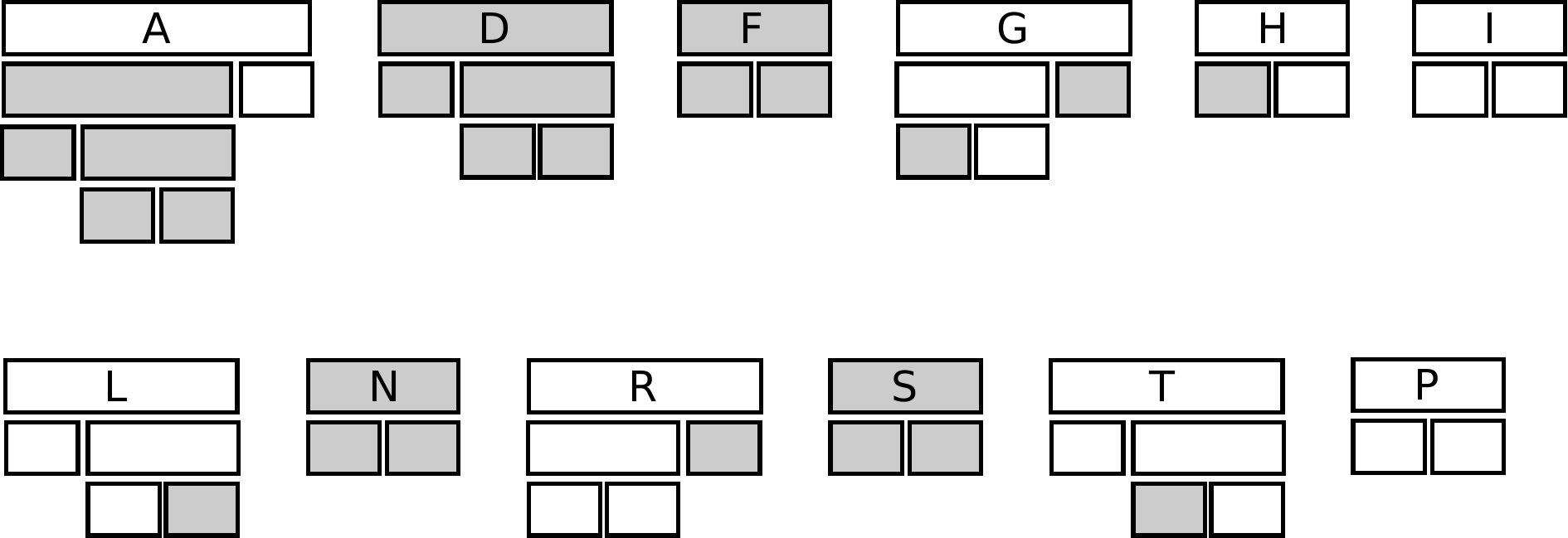}
\end{minipage}

\caption[kd-Tree 1]{The hyper-rectangles of leaf nodes (Top) from the corresponding kd-trees (Bottom). Letters show the correspondence between nodes and their hyper-rectangles. Left and Right show 28 and 41 points, respectively.
Obstacle space is red. Free nodes are gray and mixed nodes are white. Letters show the correspondence.
Descendants of a free node are always free. Mixed nodes
eventually produce free node descendants (the probability that an
obstacle node is produced is $0$). 
}

\label{fig:treeCD}
\end{figure}

We are particularly interested in the types of leaf nodes, because they cover $\totalSpace$ and also hold all of the mass that determines the induced sampling distribution.

Let $\leafNodeSet$ denote the set of leaf nodes, and let $\freeLeafNodeSet$, $\obsLeafNodeSet$, $\mixedLeafNodeSet$ denote the set of leaf nodes that are also free nodes, obstacle nodes, and mixed nodes, respectively.
We use $\spaceOf{\NodeSet} = \UnionOver{\Node \in \NodeSet} \spaceOf{\Node}$ to denote the space contained in all nodes in a set $\NodeSet$.
Figure~\ref{fig:treeCD} depicts the space contained in the set of leaf nodes, $\leafNodeSet$, of a particular kd-tree.

Recall that $\NodeMeasureFn{\Node}$ is the estimated probability mass that our algorithm associates with node $\node$.
%
Let $\deep$ denote tree depth.

\begin{proposition} \label{lemma:induced_prob}
$\probability{\SampledPoint_n \in \spaceOf{\node}} = 
    \displaystyle \NodeMeasureFn{\Node}/
        {\displaystyle \sum_{\Node' \in \leafNodeSet}  \NodeMeasureFn{\Node'} }$
\end{proposition} 
\begin{proof} 
This is true by the construction of our algorithm. In particular, from lines 16-17 and 21.
\end{proof}

\begin{proposition} \label{lemma:positive_wieght}
For all $\Node$ at depth $\deep>1$ such that  
$\measure{\spaceOf{\Node}} > 0$ there exists some $\delta > 0$ such that ${\NodeNumFreeFn{\node} > \delta}$.
\end{proposition} 
\begin{proof} 
All nodes at depth $\deep > 1$ have a parent $\NodeParentFn{\Node}$, which must have generated at least one sample in order to create $\node$. Since ${\measure{\spaceOf{\node}} > 0}$, we know that 
${\measure{\spaceOf{ \NodeParentFn{\Node} }} > 0}$. Therefore, by construction (lines 7, 11, 13) we know 
${ \NodeNumFreeFn{\Node} \geq 
   \frac{\measure{\spaceOf{ \Node }}}
        {\measure{\spaceOf{ \NodeParentFn{\Node} }}} > 0}$. 
Thus, the lemma is true for $\delta$ such that 
${0 < \delta < 
   \frac{\measure{\spaceOf{\node}}}{\measure{\spaceOf{ \NodeParentFn{\Node} }}}}$.
\end{proof}

\begin{lemma} \label{lemma:infinite_samples}
For a particular node $\Node$, let $\NumSampFn{\Node}{n}$ be the number of times that a sample was generated from $\spaceOf{\Node}$ when the kd-tree has $n$ nodes. Then, for all $\Node$ such that $\measure{\spaceOf{\Node}} > 0$, 
$\probability{\lim_{n \rightarrow \infty} \NumSampFn{\node}{n} = +\infty} = 1$.
\end{lemma} 
\begin{proof} 
We begin by obtaining two intermediate results:

First, 
$\NodeNumFreeFn{\Node} \leq \NodeNumTotalFn{\Node}$ for all $\Node$ by construction (lines 3, 5, 7, 11, 13). Thus, for all leaf nodes $ \Node \in\leafNodeSet$ it is guaranteed $ \NodeMeasureFn{\Node} \leq \measure{\spaceOf{\Node}}$ by the definition of $\NodeMeasure$ (line 12). 
Recall that the set of leaf nodes covers the space ${ \spaceOf{\leafNodeSet} = \totalSpace}$ and that the space in each leaf node is non-overlapping ${\spaceOf{\Node_i} \Intersect \spaceOf{\Node_j} = \emptyset}$ for all $\Node_i, \Node_j \in \leafNodeSet$, $\Node_i \neq \Node_j$. 
Thus, we can sum over all leaf nodes to obtain the bound: 
${\sum_{\Node \in \leafNodeSet}  \NodeMeasureFn{\Node} 
   \leq \sum_{\Node \in \leafNodeSet} \measure{\spaceOf{\Node}} 
   = \measure{\totalSpace}}$.

Second, using Proposition~\ref{lemma:positive_wieght} we know that for any particular node ${\node}$ with positive measure ${\measure{\spaceOf{\Node}} > 0}$ there exists some $\delta$ such that ${ \NodeNumFreeFn{\Node} > \delta}$.
Thus, the following bound always holds: 
${ \NodeMeasureFn{\Node} 
   = \measure{\spaceOf{\Node}} 
      \frac{ \NodeNumFreeFn{\Node} }{ \NodeNumTotalFn{\Node} } 
   \geq \measure{\spaceOf{\Node}}
      \frac{\delta}{ \NodeNumTotalFn{\Node} }}$ 
(where the first equality is by definition). Note this is the {\it worst case} situation in which node $\Node$ always samples from obstacle space (and thus $\Node$ remains a leaf node forever).
%
Thus, 
${ \NodeMeasureFn{\Node} 
   \geq \measure{\spaceOf{\Node}}
      \frac{\delta}{ \NodeNumTotalFn{\Node} }}$. 

Combining the first and second results yields:
$$
\probability{\SampledPoint_n \in \spaceOf{\Node}} 
   = \frac{ \NodeMeasureFn{\Node}}
          {\sum_{\Node' \in \leafNodeSet} \NodeMeasureFn{\Node'} } 
   \geq \frac{\delta \measure{\spaceOf{\Node}} }
         {\NodeNumTotalFn{\Node} \measure{\totalSpace}} 
$$
Where the left equality is by Proposition~\ref{lemma:induced_prob}. By definition ${\frac{\delta \measure{\spaceOf{\Node}} }{\measure{\totalSpace}} = k}$ is a constant, and so 
$
\probability{\SampledPoint_n \in \spaceOf{\Node}} 
   \geq \frac{k}{\NodeNumTotalFn{\Node}} 
$.
By definition, $\NodeNumTotalFn{\Node}$ only increase when we draw a sample from $\spaceOf{\Node}$. Let $\hat{n}$ be the iteration at which the previous sample was generated from $\spaceOf{\Node}$. 
The probability that we {\it never again} generate a sample from $\spaceOf{\Node}$ is bounded:
$${\probability{\lim_{n \rightarrow \infty} \NumSampFn{\node}{n} = \NumSampFn{\node}{\hat{n}}} \leq \lim_{n \rightarrow \infty} \prod_{i = \hat{n}}^n\left(1 - \frac{k}{\NodeNumTotalFn{\Node}} \right) = 0}$$
for all ${{\NodeNumTotalFn{\Node} < \infty}}$ and ${\NumSampFn{\node}{\hat{n}} < \infty}$ (and thus ${\hat{n} < \infty}$).
The rest of the proof follows from induction.
\end{proof}


\begin{lemma} \label{lemma:measure_zero}
Let $\freeNodeSet_n$ be the set of free nodes in the tree of $n$ samples. 
Then, for all $\Point \in \freeSpace$, $\lim_{n\rightarrow \infty}
{\probability{ \thereExists \Node \in \freeNodeSet_n \SuchThat \Point \in \closedSpaceOf{\Node} } = 1}$
\end{lemma} 
\begin{proof} 
Let $\epsBoxPoint$ be the open L1-ball with radius $\epsilon$ that is centered at point $\Point$.
For all $\Point \in \mathrm{int}(\freeSpace)$ there exists some ${\epsilon > 0}$
for which $\epsBoxPoint \subset \freeSpace$. 
Therefore, it is sufficient to prove that for $\Point \in \freeSpace$, 
$\lim_{n\rightarrow \infty}
{\probability{ \thereExists \Node \in \freeNodeSet_n \SuchThat \Node
      \subset \epsBoxPoint} = 1}$.
  
%
Without loss of generality, we now consider a particular $\Point$.
At any point during the run of the algorithm there is some leaf node 
${\leafNode \SuchThat \leafNode \owns \Point}$.

Lemma~\ref{lemma:infinite_samples} guarantees that ${\leafNode \owns \Point}$ will almost surely split into two children, 
one of which will also contain $\Point$, etc. Let $\node_{\deep, \Point}$ represent the node at depth $\deep$ that contains $\Point$. Let ${\SampledPoint_\deep \in \spaceOf{\node_{\deep, \Point}}}$ be the sample point that causes $\node_{\deep, \Point}$ to split. Let $\PointCoord{\Point}{i}$ refer to the $i$-th coordinate of $\Point$. 
Thus, the splitting plane is normal to the ${\deep \bmod{\Dimension}}$-axis, and intersects that axis at ${\SampledPoint_\deep[\deep \bmod{\Dimension}]}$, where $\Dimension$ is the dimensionality of the space.

Each time the current leaf ${\node_{\deep, \Point}\owns\Point}$ splits
${\probability{
   \!\SampledPoint_\deep \in \epsBoxPoint 
   \land \PointCoord{\SampledPoint_\deep}{i} < \PointCoord{\Point}{i} } 
   \!=\! \frac
      {\measure{\spaceOf{\node_{\deep, \Point}} \Intersect \spaceOf{\epsBoxPoint}}}
      {2\measure{\spaceOf{\node_{\deep, \Point}}\Intersect \freeSpace}} 
   \!>\! 0}$. By construction
${\frac
   {\measure{ \spaceOf{ \node_{\deep+\Dimension, \Point}} 
      \Intersect \spaceOf{\epsBoxPoint}}}
   {\measure{\spaceOf{\node_{\deep+\Dimension, \Point}}
      \Intersect \freeSpace}} 
   \geq \frac
   {\measure{\spaceOf{\node_{\deep, \Point}}
      \Intersect \spaceOf{\epsBoxPoint}}}
   {\measure{\spaceOf{\node_{\deep, \Point}} \Intersect \freeSpace}}}$
so
%
{\small
$\displaystyle{\lim_{\deep \rightarrow \infty} \probability{
      \! \thereExists \SampledPoint_\deep 
      \SuchThat \SampledPoint_\deep  \in  \epsBoxPoint 
      \land \PointCoord{\SampledPoint_\deep}{\deep \bmod{\Dimension}}  
      <  \PointCoord{\Point}{\Dimension \bmod{\deep}} }\! =\! 1}$.
}%
A similar argument can be made for ${ \PointCoord{\SampledPoint}{i} > \PointCoord{\Point}{i} }$,
{\small
$\displaystyle{\lim_{\deep \rightarrow \infty} \probability{
   \! \thereExists \SampledPoint_\deep 
   \SuchThat \SampledPoint_\deep  \in  \epsBoxPoint 
   \land \PointCoord{\SampledPoint_\deep}{\deep \bmod{\Dimension}}   
   > \PointCoord{\Point}{\deep \bmod{\Dimension}} } \!=\! 1}$%
}.
%
%
Thus, in the limit as ${\deep \rightarrow \infty}$, there will almost surely be a set of $2 \Dimension$ points 
$\{ \SampledPoint_{\deep_{1}}, \ldots, \SampledPoint_{\deep_{2\Dimension}} \}$ 
sampled at levels $\deep_1, \ldots, \deep_{2\Dimension}$, 
such that ${\SampledPoint_{\deep_i} \in \epsBoxPoint}$ for ${i = \{1, \ldots, 2\Dimension\}}$, and 
${i = \deep_i \bmod{\Dimension}}$ and
${ \PointCoord{ \SampledPoint_{\deep_i} }{i} < \PointCoord{\SampledPoint}{i} }$, and 
${i = \deep_{d+i} \bmod{\Dimension}}$ and
${ \PointCoord{ \SampledPoint_{\deep_{d+i}} }{i} > \PointCoord{\SampledPoint}{i} }$. 
By construction, $\SampledPoint_{\max_{i}(\deep_{i})}$ is on a splitting plane that borders a node ${\Node}$ such that ${\spaceOf{\Node} \owns \Point}$ and ${\Node \subset \epsBoxPoint}$ (and thus ${\Node \in \freeNodeSet}$). 
Lemma~\ref{lemma:infinite_samples} implies that ${\probability{\lim_{n\rightarrow\infty} \deep = \infty} = 1}$ for $\node_{\deep, \Point} \SuchThat \spaceOf{\node_{\deep, \Point}} \owns \Point$.
\end{proof}

\begin{corollary}  \label{corollary:zero_space}
${\probability{\lim_{n \rightarrow \infty} \measure{  \spaceOf{\mixedLeafNodeSet} \Intersect \freeSpace } = 0} = 1}$.
\end{corollary}

\begin{corollary} \label{corollary:b}
${\probability{\lim_{n \rightarrow \infty} \measure{ \freeSpace \setminus \spaceOf{ \freeLeafNodeSet } } = 0} = 1}$.
\end{corollary}

\begin{corollary} \label{corollary:c}
\verb| |\\
$\probability{\lim_{n \rightarrow \infty} 
   \sum_{\node \in \freeLeafNodeSet} \measure{\spaceOf{\node}} 
   = \measure{\freeSpace}} = 1$.
\end{corollary}

\begin{lemma} \label{lemma:zero_mass_obstacle} 
${\probability{\lim_{n \rightarrow \infty} \NodeMeasureFn{\Node} = 0} = 1}$
for all obstacle leaf nodes $\Node \in \obsLeafNodeSet$.
\end{lemma} 
\begin{proof}
There are two cases, one for ${\measure{\spaceOf{\Node}} = 0}$ and another for ${\measure{\spaceOf{\Node}} > 0}$. The first is immediate given $\NodeMeasureFn{\Node} \DefinedAs 
   \frac{ \NodeNumFreeFn{\Node} }{ \NodeNumTotalFn{\Node}} 
   \measure{\spaceOf{\Node}}$. 
For the second, we observe that
$\probability{ \exists \Point \SuchThat \Point \in \spaceOf{\Node} \land \Point \in \freeSpace} = 0$ by definition, and so $\NodeNumFreeFn{\Node}$ will almost surely not change (and $\Node$ will remain a leaf node almost surely). Thus, ${\probability{ \NodeNumTotalFn{\Node} = \infty} = 1}$ by Lemma~\ref{lemma:infinite_samples}, and so 
${\probability{\lim_{n \rightarrow \infty}
   \frac{ \NodeNumFreeFn{\Node} }{ \NodeNumTotalFn{\Node} } = 0} = 1}$. 
Using the definition of $\NodeMeasureFn{\Node}$ finishes the proof.
\end{proof}

\begin{corollary}  \label{corollary:prob_zero_obs}
$\lim_{n \rightarrow \infty}  \probability{ \SampledPoint_n \in \spaceOf{\obsLeafNodeSet} } = 0$.
\end{corollary} 

\begin{lemma} \label{lemma:zero_mass_mixed} 
${\probability{\lim_{n \rightarrow \infty} \NodeMeasureFn{\Node} = 0} = 1}$ for all mixed leaf nodes $\Node \in \mixedLeafNodeSet$.
\end{lemma} 
\begin{proof}
 $\NodeNumFreeFn{\Node}$ will almost surely not change by Corollary~\ref{corollary:zero_space}. The rest of the proof is similar to Lemma~\ref{lemma:zero_mass_obstacle}.
\end{proof}

\begin{corollary}  \label{corollary:prob_zero}
$\lim_{n \rightarrow \infty} \probability{\SampledPoint_n \in \spaceOf{\mixedLeafNodeSet} } = 0$.
\end{corollary} 

\begin{corollary} \label{corollary:no_sample_obs_or_mixed}
\verb| | \\
${\lim_{n \rightarrow \infty} \probability{
\SampledPoint_n \in \left( \spaceOf{\mixedLeafNodeSet}  \Union  \spaceOf{\obsLeafNodeSet}  } \right) = 0}$
\end{corollary}

We observe that this result does not conflict with Lemma~\ref{lemma:infinite_samples}. Each node with finite space is sampled an infinite number of times; however, the proportion of samples from obstacle nodes and mixed nodes approaches 0 in the limit as $n \rightarrow \infty$.

\begin{lemma} \label{lemma:case_1}
${\lim_{n \rightarrow \infty} \probability{ \SampledPoint_n \in \obsSpace} = 0}$
\end{lemma} 
\begin{proof} 
This follows from Corollary~\ref{corollary:no_sample_obs_or_mixed} and the fact that 
$\measure{\obsSpace \setminus 
   \left( \spaceOf{\mixedLeafNodeSet} 
      \Union  \spaceOf{\obsLeafNodeSet}  \right)} = 0$.
\end{proof}

\begin{lemma} \label{lemma:free_nodes_weight}
$\probability{\lim_{n \rightarrow \infty} \NodeMeasureFn{\Node} = \measure{\spaceOf{\Node}}} = 1$, for all free nodes $\Node \in \freeNodeSet$.
\end{lemma} 
\begin{proof} 
There are two cases, one for when $\measure{\spaceOf{\Node}} = 0$ and another for when $\measure{\spaceOf{\Node}} > 0$. The former is immediate given the definition of $\NodeMeasureFn{\Node}$, and so we focus on the latter. When a new free node $\Node_\deep \in \freeNodeSet$ is created at depth $\deep > 1$ of the tree it initializes $\NodeNumFreeFn{\Node_\deep} > 0$ and  $ \NodeNumTotalFn{\Node_\deep} > 0$ based on similar values contained in its parent (and wighted by the relative measures of $\Node_\deep$ vs. its parent). By Lemma~\ref{lemma:infinite_samples} we know that $\Node_\deep$ will almost surely generate two children $\Node_{\deep+1, 0}$ and $\Node_{\deep+1, 1}$. By construction (lines 11-13), they will be initialized with ${\Node_{\deep+1, j}.\NumFree = (\Node_\deep.\NumFree + 1) \frac{\measure{\spaceOf{\Node_{\deep+1, j}}}}{\measure{\spaceOf{\Node_\deep}}}}$ and ${\Node_{\deep+1, j}.\NumSamples = (\Node_\deep.\NumSamples + 1) \frac{\measure{\spaceOf{\Node_{\deep+1, j}}}}{\measure{\spaceOf{\Node_\deep}}}}$, for ${j \in \{0 , 1 \}}$. These children will also generate their own children almost surely, etc. Because $\Node_\deep$ is a free node, all samples from its sub-tree will result in more {\it free node} descendants being created almost surely. Let $\freeDescendents$ be the set containing all leaf node descendants of $\Node_\deep$ at iteration $n$. By construction (line 24), as soon as $| \freeDescendents | \geq 1$, then  
${\NodeMeasureFn{\Node_\deep} = 
   \sum_{\Node \in \freeDescendents} \NodeMeasureFn{\Node} }$.  
We now examine a single term of the latter summation, i.e., the term for node $\Node_{\deep+k}$ at depth ${\deep + k}$. In particular.
${\Node_{\deep+k}.\distMass = \frac{\Node_{\deep+k}.\NumFree}{\Node_{\deep+k}.\NumSamples} \measure{\spaceOf{\Node_{\deep+k}}}}$.
For the remainder of this proof we will abuse our notation and let ${\|\cdot\| = \measure{\spaceOf{\cdot}}}$ to make the following equations more readable.
Unrolling the recurrence relation for $\Node_{\deep+k}.\NumFree$ gives:
{\small
\begin{align*}
&\Node_{\deep+k}.\NumFree = \\ 
& {\frac{\|\Node_{\deep+k}\|}{\|\Node_{\deep+k-1}\|}\!\!\left(\ldots \frac{\|\Node_{\deep+2}\|}{\|\Node_{\deep+1}\|}\left(\frac{\|\Node_{\deep+1}\|}{\|\Node_{\deep}\|}\left(\Node_{\deep}.\NumFree + 1\right) + 1 \!\!\right) \ldots + 1 \right)}
\end{align*}
}
where $\Node_{\deep+k-1}, \ldots, \Node_{\deep+2}, \Node_{\deep+1}, \Node_{\deep}$, are the ancestors of $\Node_{\deep+k}$ going up the tree to 
$\Node_{\deep}$.
This can be rearranged:
{
\small
\begin{align*}
\Node_{\deep+k}&.\NumFree = \\
&\frac{\|\Node_{\deep+k}\|}{\|\Node_{\deep}\|}\Node_{\deep}.\NumFree + \frac{\|\Node_{\deep+k}\|}{\|\Node_{\deep}\|} + \frac{\|\Node_{\deep+k}\|}{\|\Node_{\deep+1}\|} + \ldots +  \frac{\|\Node_{\deep+k}\|}{\|\Node_{\deep+k-1}\|}
\end{align*}
}%
Similarly, the $\Node_{\deep+k}.\NumSamples$ recurrence relation is:
{
\small
\begin{align*}
\Node_{\deep+k}&.\NumSamples = \\
&\frac{\|\Node_{\deep+k}\|}{\|\Node_{\deep}\|}\Node_{\deep}.\NumSamples + \frac{\|\Node_{\deep+k}\|}{\|\Node_{\deep}\|} + \frac{\|\Node_{\deep+k}\|}{\|\Node_{\deep+1}\|} + \ldots +  \frac{\|\Node_{\deep+k}\|}{\|\Node_{\deep+k-1}\|}
\end{align*}
}%
${\lim_{k \rightarrow \infty} \frac{\|\Node_{\deep+k}\|}{\|\Node_{\deep}\|} = 0}$, also ${\probability{\frac{\|\Node_{\deep+k}\|}{\|\Node_{\deep+k-1}\|} = 0} = 0}$ given ${\measure{\spaceOf{\Node_{\deep}}} > 0}$, where we resume our normal notation. 
Thus, ${\probability{\lim_{k \rightarrow \infty} \Node_{\deep+k}.\distMass  = \measure{\spaceOf{\Node_{\deep+k}}}} = 1}$.

Lemma~\ref{lemma:infinite_samples} guarantees that ${\probability{{\lim_{n\rightarrow \infty} k = \infty}} = 1}$ for all $\Node_{\deep+k}$ such that ${\freeDescendents \owns \Node_{\deep+k}}$.
Thus, by summing over the members of $\freeDescendents$ we get:
${\probability{{\lim_{n \rightarrow \infty} 
\sum_{\Node \in \freeDescendents} \Node.\distMass = 
\sum_{\Node \in \freeDescendents} \measure{\spaceOf{\Node}}}
} = 1}$. 
${\Node_\deep.\distMass = \sum_{\Node \in \freeDescendents} \Node.\distMass}$ by definition. Also by definition ${\spaceOf{\Node_\deep} = \bigcup_{\Node \in \freeDescendents} \spaceOf{\Node}}$
and $\Node_i \cap \Node_j = \emptyset$ for all ${\Node_i, \Node_j \in \freeDescendents}$ such that ${i\neq j}$; therefore,
${\measure{\spaceOf{\Node_\deep}} = \sum_{\Node \in \freeDescendents} \measure{\spaceOf{\Node}}}$.
Substitution finishes the proof.
\end{proof}

Note, Corollary~\ref{corollary:a} depends on Lemma~\ref{lemma:free_nodes_weight} and Corollary~\ref{corollary:b}:

\begin{corollary} \label{corollary:a}
\verb| |\\
$\probability{\lim_{n \rightarrow \infty} \sum_{\node \in \freeLeafNodeSet} 
\NodeMeasureFn{\Node} = \measure{\freeSpace}} = 1$.
\end{corollary}

\begin{lemma} \label{lemma:case_3}

${\probability{\lim_{n \rightarrow \infty}\probDensity{n}{x} = c} = 1}$ for all $x$ and $B$  such that ${x \in B \subset \freeSpace}$ and $\measure{B} > 0$

\end{lemma} 
\begin{proof} 

By Proposition~\ref{lemma:induced_prob} and Lemmas \ref{lemma:free_nodes_weight} and Corollary~\ref{corollary:a} we know that ${\lim_{n \rightarrow \infty} \probability{\Point_n \in \spaceOf{\Node}} = \frac{\measure{\spaceOf{\Node}}}{ \measure{\freeSpace}  }}$ for all free nodes  ${\Node \in \freeNodeSet}$ almost surely. By construction (\mbox{line 2}) once a leaf node ${\Node \in \freeLeafNodeSet}$ is reached, samples are drawn uniformly from within $\spaceOf{\Node}$.
Thus, the uniform probability density of drawing ${\Point_n \in \spaceOf{\Node}}$, given that the algorithm has decided to draw from within $\spaceOf{\Node}$, is ${\probDensity{n}{\Point_n | \Point_n \in \spaceOf{\Node}} = \frac{1}{\measure{\spaceOf{\Node}}}}$.
Therefore, the posterior probability density $\lim_{n \rightarrow \infty}\probDensity{n}{\Point_n} =
 {\lim_{n \rightarrow \infty} \probDensity{n}{\Point_n | \Point_n \! \in \!\spaceOf{\Node}} \probability{\Point_n \! \in \spaceOf{\Node}} \! = \! \frac{1}{\measure{\freeSpace}}}$ almost surely, 
which is constant and {\it independent} of $\Node \in \freeLeafNodeSet$, and thus holds almost everywhere in $\UnionOver{\Node \in \freeLeafNodeSet} \spaceOf{\Node}$---and thus almost everywhere in $\freeSpace$ (by Corollary~\ref{corollary:b}). 
\end{proof}

\begin{theorem} \label{theorem:sampling}
$\probability{\lim_{n \rightarrow \infty}\probDensity{n}{x} = \probUniformFree{x}} = 1$.
\end{theorem}
\begin{proof}
This is proved by combining Lemmas~\ref{lemma:case_1} and \ref{lemma:case_3}.
\end{proof}

\section{Summary}

This document contains improved and updated proofs of convergence for the sampling method presented in our paper ``Free-configuration Biased Sampling for Motion Planning'' \cite{bialkowski_IROS13}.

\bibliographystyle{plain}
\bibliography{publications}

\end{document}